\documentclass[letterpaper]{article} 
\usepackage{aaai24}  
\usepackage{times}  
\usepackage{helvet}  
\usepackage{courier}  
\usepackage[hyphens]{url}  
\usepackage{graphicx} 
\usepackage{natbib}  
\usepackage{caption} 
\usepackage[ruled,linesnumbered, noend]{algorithm2e}
\usepackage[table,xcdraw]{xcolor}
\SetAlgoSkip{}
\usepackage{newfloat}
\usepackage{listings}
\DeclareCaptionStyle{ruled}{labelfont=normalfont,labelsep=colon,strut=off} 
\newcommand\blfootnote[1]{%
  \begingroup
  \renewcommand\thefootnote{}\footnote{#1}%
  \addtocounter{footnote}{-1}%
  \endgroup
}

\usepackage{subcaption}

\title{A Data Efficient Framework for Learning Local Heuristics}
\author{
    Rishi Veerapaneni\equalcontrib,
    Jonathan Park\equalcontrib,
    Muhammad Suhail Saleem,
    Maxim Likhachev
}
\affiliations{
    Carnegie Mellon University \\
    \{rveerapa, jkp2, msaleem2, mlikhach\}@andrew.cmu.edu
}

\usepackage{hhline}
\usepackage{amsmath}
\usepackage{multirow}
\usepackage{placeins}
\usepackage{amssymb}
\usepackage{amsthm}

\newtheorem{theorem}{Theorem}

\begin{document}

\maketitle

\begin{abstract}
With the advent of machine learning, there have been several recent attempts to learn effective and generalizable heuristics. Local Heuristic A* (LoHA*) is one recent method that instead of learning the entire heuristic estimate, learns a ``local" residual heuristic that estimates the cost to escape a region. 
LoHA*, like other supervised learning methods, collects a dataset of target values by querying an oracle on many planning problems (in this case, local planning problems). This data collection process can become slow as the size of the local region increases or if the domain requires expensive collision checks. Our main insight is that when an A* search solves a start-goal planning problem it inherently ends up solving multiple local planning problems. We exploit this observation to propose an efficient data collection framework that does $<$1/10th the amount of work (measured by expansions) to collect the same amount of data in comparison to baselines. 
This idea also enables us to run LoHA* in an online manner
where we can iteratively collect data and improve our model while solving relevant start-goal tasks.
We demonstrate the performance of our data collection and online framework on a 4D $(x, y, \theta, v)$ navigation domain. 

\end{abstract}

\section{Introduction}

Search-based planning approaches are widely used in various robotics domains from navigation to manipulation. However, their runtime performance is highly dependant on the heuristics employed. A vast majority of search-based planning methods rely on hand-designed heuristics which are usually geometry-based (e.g. Euclidean distances) or solve a lower-dimensional projection of the problem (ignoring some robot/environment constraints). Since these heuristics are simple and manually defined, they can guide search into deep minima \cite{KORF1990189,mha2014}.

Recently, several works have proposed using machine learning to obtain better heuristics (or priorities) to speed up search. We focus on methods using supervised learning which requires collecting a training dataset of optimal solutions generated by an oracle search method. 
For example, \citet{learningHeuristicA2020} uses A* to generate a training dataset of optimal solution values while \citet{Takahashi_Sun_Tian_Wang_2021} uses a backward Dijkstra. 
\citet{learnExpansionDelay2021} learns an ``expansion delay" heuristic (a proxy measure for the size of local minima) and gathers training data by using an oracle A* on problems and recording the expansion number for each state on the optimal path. 
SAIL \cite{sail2017} learns a priority function by utilizing optimal cost-to-go values collected via a backward Dijkstra oracle deployed solely for data collection. 

Local Heuristic A* (LoHA*) \cite{localHeuristic} is a recent promising work that learns a residual ``local" heuristic. In contrast to a global heuristic which estimates the cost to reach the goal from a state, a residual local heuristic estimates the additional cost required to reach the border of a local region surrounding that state. 
As local heuristics require reasoning only about small regions, they are much easier to learn and generalize better.

However, similar to other supervised learning approaches, LoHA* requires creating a dataset of ground-truth local heuristic residuals. These residuals are computed by defining a local region around a state and running a multi-goal A* from it (details in Section \ref{sec:looking-back}). Thus to collect their dataset, similar to other works, they require running multiple (thousands) of oracle A* calls for training their model. We make the key observation that when an A* search solves a ``global" start-goal planning problem, the inherent best-first ordering associated with the state expansions enables a single A* query to automatically solve multiple local heuristic problems without the need to explicitly query a local search. 
Thus during the global A* call, we design backtracking logic that verifies if a state being expanded is a solution to a local planning problem and adds it to a training dataset if so. This allows us to create a significantly more efficient data collection framework wherein solving a handful of global planning problems allows us to amass sufficient training data.

An important byproduct of this data collection mechanism is that we no longer need a separate data collection phase and can collect data when running LoHA*. Since LoHA* runs a search method during test time, we can collect local heuristic data \textit{online while using LoHA* itself}. This enables us to rapidly learn from our experiences 
in a way not possible with other data collection techniques. \blfootnote{See https://arxiv.org/abs/2404.06728 for appendix.}

Thus overall, our main technical contribution is Data Efficient Local Heuristic A* (DE-LoHA*), our efficient backtracking technique for collecting local heuristic data which also enables online learning. 
We show how this improves data collection efficiency by 10x and how our online method can learn from experiences in under 100 planning calls on a 4D $(x, y, \theta, v)$ navigation domain. 

\begin{figure}[t]
    \centering
    \includegraphics[width=0.45\textwidth]{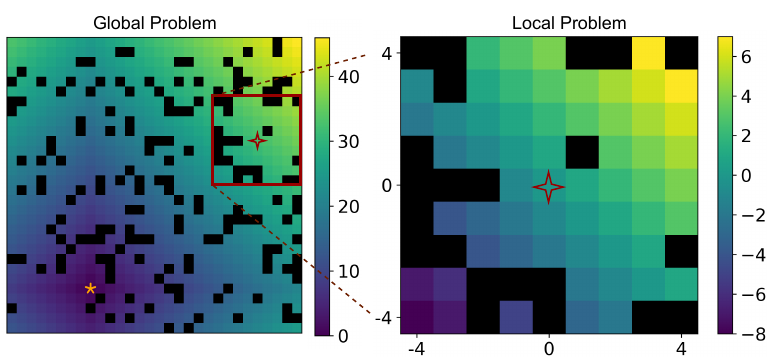}
    \vspace{-0.5em}
    \caption{Figure borrowed from LoHA* \cite{localHeuristic}. Instead of estimating the entire cost-to-go from state $s$ (red diamond) to the goal (orange in left), LoHA* computes the residual cost for $s$ to reach a border region (red box in left, zoomed in on right). This avoids local minima when used during search.
    }
    \label{fig:lh-related-work}
    \vspace{-1em}
\end{figure}

\section{Preliminaries}

Given a planning domain and a simple hand-designed heuristic (like Euclidean distance), which we will call a ``global" heuristic $h_g$, LoHA* proposes to learn a local heuristic residual $h_k$.
Formally, given a state $s = (x,y,\Omega)$ with position $x,y$ and other state parameters $\Omega$ (e.g. heading, velocity), a local region $LR(s)$ contains the states within a window of $K$, i.e. $LR(s) = \{s' \mid K \geq |s.x-s'.x|, K \geq |s.y-s'.y|\}$. Let $LRB(s)$ be the border of this region, i.e. $\{s' \mid K = |s.x-s'.x| \vee K = |s.y-s'.y| \}$. 
Any path from $s$ to $s_g$ must contain a state in $LRB(s)$, or directly reach the goal in the local region $LR(s)$ (we assume for simplicity unit actions, but this logic can be generalized to non-unit as well). If neither is possible from $s$, then $s$ cannot leave $LR(s)$ and should have an infinite heuristic value. Let, 
\begin{equation} \label{equation:localH}
h_{gk}(s) = \min_{s'}
\begin{cases}
    c(s,s')+h_g(s') & s' \in LRB(s) \\
    c(s,s')+0, & s' = s_g \in LR(s)\\
    \infty, & \text{otherwise}
\end{cases}
\end{equation}
\vspace{-0.5em}

\noindent Then, the local heuristic residual $h_{k}(s)$ is defined as $h_k(s) = h_{gk}(s) - h_g(s)$. 
$h_{gk}$ is a more informed heuristic as it accurately accounts for the cost to reach a border state from the current state $s$. Conceptually, the defined local heuristic residual captures the mismatch in the estimate by $h_g$ of the cost to reach $LRB(s)$ and the actual cost it takes.

Computing $h_k(s)$ requires identifying the best border state $s' \in LRB(s)$ that minimizes $h_{gk}(s)$. This is achieved by running a local (multi-goal) A* search from $s$ using all the border states as goals with $h_g$ as the heuristic ($h_g$ estimates the cost to $s_g$ and not the border states). In the most common case, where the goal is not located within $LR(s)$, the search terminates upon expanding the first state in $LRB(s)$. This heuristic residual is effective but slow to compute, so LoHA* approximates it with a neural network. 

This network takes in observations in the local region $LR(s)$, namely a local image of the obstacle map and $h_g$ values to predict $h_k(s)$. 
They collect a dataset offline and regress their network to predict ground truth $h_k(s)$ values.
This model when used with Focal Search \cite{focalSearch1982} has been shown to substantially reduce the number of nodes expanded compared to using just $h_g$ while maintaining suboptimality guarantees.


\begin{figure*}[t]
    \centering
    \includegraphics[width=0.78\textwidth]{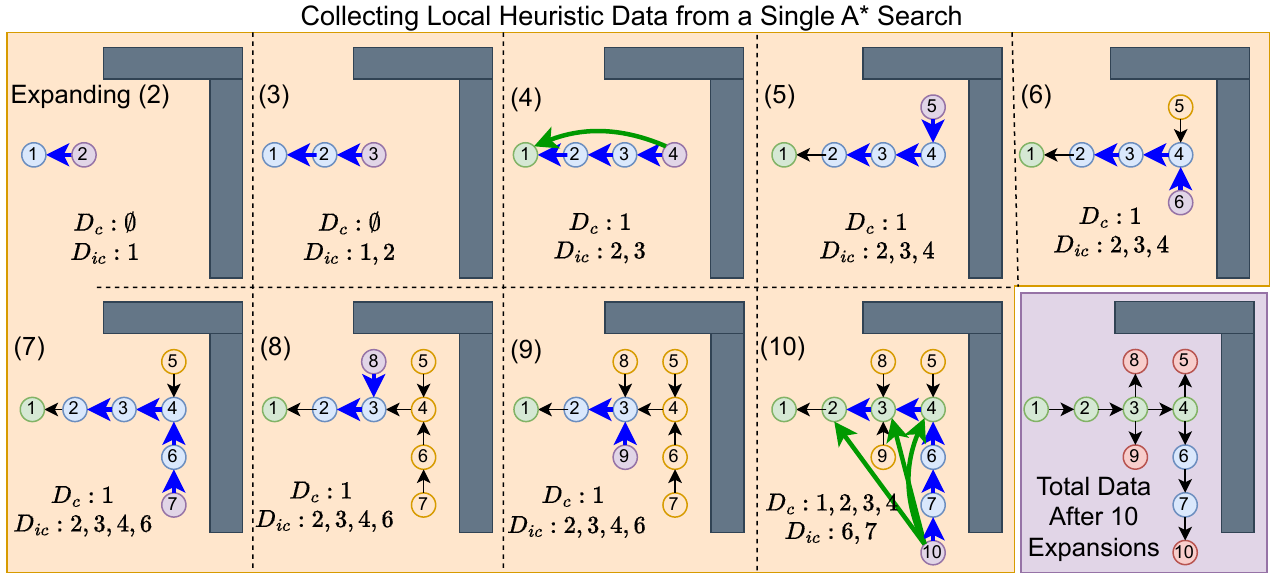}
    \vspace{-2pt}
    \caption{A simplified example of global A* collecting data for local regions with $K=3$, i.e. for $s=(x,y)$ we want to reach a state $(x',y')$ with $x'$ or $y'$ 3 away. Each (i) depicts the $i^{th}$ state expanded with successor $\rightarrow$ parent denoted. Expanding (2), (3) leads to incomplete data $D_{ic}$ for $s_1, s_2$ as we made some progress but did not reach a $LRB$ yet. Expanding $s_4$ and backtracking reveals that $s_4 \in LRB(s_1)$; we have found $LH(s_1)$ and add it to our complete dataset $D_c$. We continue to expand a node (purple), backtrack (blue arrows) to update values of ancestors (blue) whose $LH(s')$ have not been computed. Bottom right: After 10 expansions, we collect 4 complete (green) and 2 partial (blue) LH values, and cannot collect data for leaf nodes (red). 
    }
    \label{fig:overview}
    \vspace{-1em}
\end{figure*}

\section{Data Efficient Local Heuristic}
To train the neural network for LoHA*, a dataset of ground truth local heuristic residuals is collected by running the above-described local search on hundreds of thousands of relevant states. This is extremely inefficient and can become prohibitively slow for some domains that require expensive collision checks or large local regions. Our main observation is that the inherent best-first ordering associated with A* expansions enables a single global A* call to solve multiple local heuristic problems. We specifically design backtracking logic that attempts to gather a data point at every single A* expansion (without the need for running a local A*), enabling us to reuse prior search efforts and collect data at a significantly faster rate.

\subsection{Collecting Data by Looking Back}
\label{sec:looking-back}

Our method is built upon a crucial observation regarding the correlation between the order in which nodes are prioritised/expanded in a global search and a local multi-goal A*. \textit{Upon expanding a node $s$ in a global search, the relative ordering in which the nodes in $LR(s)$ (originating from $s$) are expanded in the global search is identical to the order in which the nodes would be expanded in a local search from $s$.} Here, nodes originating from $s$ refer to nodes that have $s$ as a direct parent or ancestor. If this observation can be proven to be true, then the first node in $LRB(s)$ (and originating from $s$) expanded by the global search will correspond to the best border child of $s$ and can be used to compute the local residual for $s$. For notational convenience, let $b(s, s') = c(s, s') + h_g(s')$.

\begin{theorem}[Global-Local Ordering Consistency]
\label{thm:ordering}
A local A* using priority $b(s,s')$ and a global A* using $b(s_{start},s')$ will sort states originating from $s$ in $LR(s)$ identically.
\end{theorem}

\begin{proof}
    Once $s$ is expanded in the global A*, all successor states $s'$ (direct children and onwards) are prioritized by $b(s_{start},s') = c(s_{start},s) + c(s, s') + h_g(s') = c(s_{start},s) + b(s,s')$. Since $c(s_{start},s)$ is a constant, we can remove this term without changing the sorted order to get that the children are sorted by $b(s, s')$. This is identical to the local A* search ordering rooted at $s$.
\end{proof}


Thus now, instead of running a local search rooted at different states, we can simply verify if the node $s'$ being expanded by the global search is the first in the $LRB$ for some previously expanded node $s$. We achieve this by simply backtracking from $s'$ to each of its ancestors $s$ and checking if $s' \in LRB(s)$ and if any node in $LRB(s)$ has previously been expanded. If not, we are guaranteed that $s'$ is the best border state for $s$ and $h_k(s)$, the local heuristic residual for $s$ is computed and included in the dataset. This verification process is extremely fast and has negligible overhead as it requires no collision checking or queue operations allowing the global search to operate unhindered. 
This idea is similarly motivated as Hindsight Experience Replay \cite{hindsight2017} but is a more nuanced idea leveraging best-first search's intermediate expansion process.

Figure \ref{fig:overview} shows this process occurring over expansions in a global search. When expanding state $s'$ (purple), we backtrack (blue arrows) ancestor states $s$ and check if $s' \in LRB(s)$. We see this first occurs when $s_4$ satisfies $LRB(s_1)$, denoted by the green arrow. Since $s_1$ has found its best border state, we do not include it in future checks (as seen when expanding $s_5$ we do not backtrack to $s_1$). Repeating this logic, after 10 expansions we collect 4 data points. 

\subsection{Collecting Partial Data}
Although this logic works, the amount of data collected per iteration of global search was lower than expected. This was due to many nodes $s$ in the global search never having an $s' \in LRB(s)$ expanded. Instead, many nodes had made partial progress, i.e. nodes in the local region but not in the border were expanded. In Figure \ref{fig:overview} this corresponds to $s_6$ and $s_7$ which both have partial progress (via $s_{10}$) but don't have a node in the $LRB$ expanded. Instead of completely ignoring these data points, we developed a mechanism to extract useful approximations on $h_k(s)$ and include it in the dataset. 

We know that in A* the priority $(b_{start}, s')$ of expanded states $s'$ increases monotonically until the goal is reached. Hence, at any point during an A* search the priority of an expanded $s'$ is a lower bound on the optimal solution cost. In our case the priority of the state $s' \in LR(s)\backslash LRB(s)$ ends up defining a lower bound on $h_{gk}(s)$. Thus, each time we expand $s'$, we update the $h_{gk}(s) \leftarrow c(s, s') + h_g(s')$ of each ancestor of $s$ whose $LRB(s)$ has not been reach yet. 
Upon the global search terminating, we can use these lower bound values which leads to dramatically more datapoints.  


\begin{figure*}[t!]
    \begin{subfigure}[c]{0.34\textwidth}
        \resizebox{1\textwidth}{!}{
        \centering
        \begin{tabular}{|c|c|c|c|c|c|} 
        \cline{2-6}
        \multicolumn{1}{c|}{} & \multicolumn{5}{c|}{\# Expansions Per Sample} \\ \hline
        $K=$ & 2 & 4 & 8 & 12 & 16 \\ \hline
        Local A* & 11 & 74.4 & 247 & 458 & 616 \\
        Complete & 16.5 & 27.1 & 34.9 & 37.6 & 38.9 \\
        Incomplete & 5.0 & 5.0 & 5.0 & 5.0 & 5.0 \\ \hline
        \end{tabular}
        }
        \caption{Amount of work (\# expansions) required to collect a single data point as $K$ increases. Local A* uses an oracle A* to produce individual data points while Complete and Incomplete gather data across global start-goal problems using our backtracking methodology.}
        \label{table:expansions-per}
    \end{subfigure}
    \hfill
    \begin{subfigure}[c]{0.32\textwidth}
         \centering
         \includegraphics[width=\textwidth]{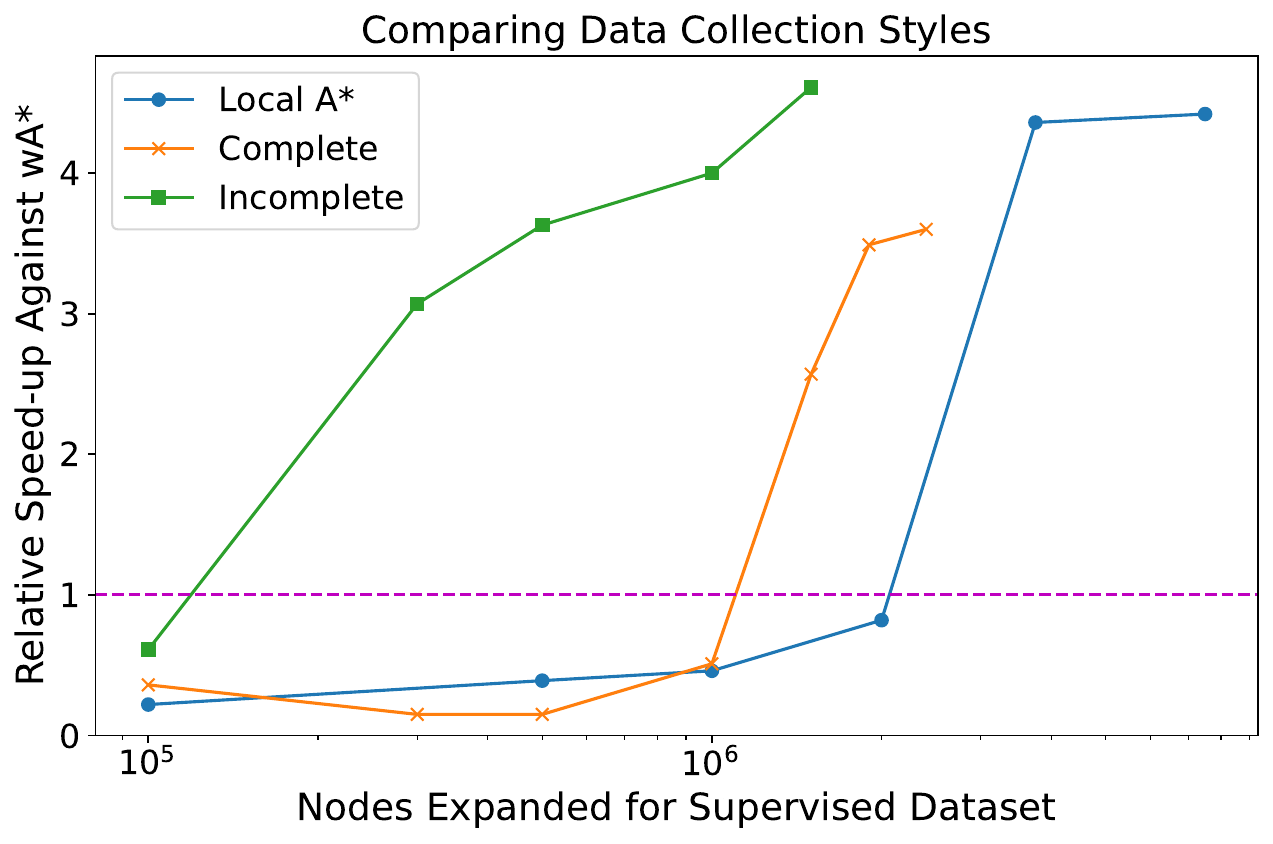}
         \caption{Performance vs Work for Dataset}
         \label{fig:offline_performance}
    \end{subfigure}
    \begin{subfigure}[c]{0.32\textwidth}
         \centering
         \includegraphics[width=\textwidth]{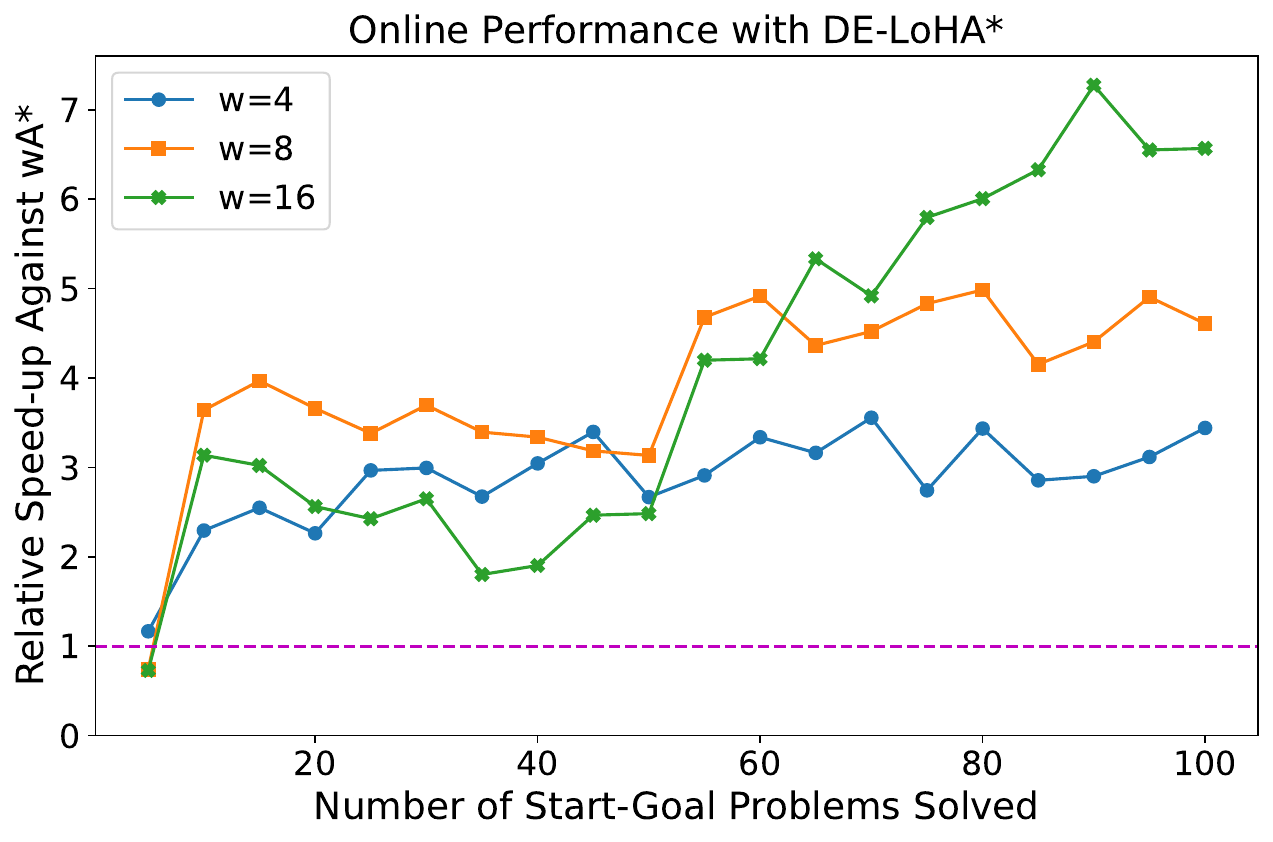}
         \caption{Online performance of DE-LoHA*}
         \label{fig:online_performance}
    \end{subfigure}
    \caption{(a) compares collecting data via a ground truth oracle (Local A*) against our ``Complete" and ``Incomplete" collection methods. 
    (b) plots the ``Speed-up" (as measured in nodes expanded) of LoHA* trained on datasets gathered from Local A* calls (true oracle) or our datasets from running a global A* on start-goal problems. 
    (c) We run DE-LoHA* by solving start-goal problems (and collecting data while doing so), and retraining every 5 problems. We see that DE-LoHA* can improve performance from just solving start-goal problems without needing an external oracle after the initial 5 problems. }
    \label{fig:noise-both}
    \vspace{-1em}
\end{figure*}


One drawback of including this ``incomplete" data is that our dataset gets skewed as these points are more numerous than complete data. We handled this issue by downweighting their contribution to the loss function. Concretely, let $s'$ correspond to the last expanded node in the local region of state $s$. $s'$ was used for computing the lower bound on $h_{gk}(s)$. Let $d(s,s') = \max(|s.x-s'.y|, |s.y-s'.y|)$ be the distance from $s'$ to $s$. Since $s' \in LR(s) \backslash LRB(s)$, we have $d(s,s') < K$. We can interpret $\alpha(s,s') = d(s,s')/K < 1$ as the progress towards reaching the border of $s$. Thus when training and regressing onto $LH(s)$, we weigh our loss by $\alpha(s,s')$. Intuitively, this downweighs ``incomplete" data points by their progress but does not ignore them entirely.

\textbf{Closed List as an Obstacle}
Astute readers may have noticed a possible issue with Theorem \ref{thm:ordering}; although the states are sorted identically, we ignored the effect of the closed list of previously expanded states in the global A*. These states $\in LR(s)$ would \textit{not} be re-expanded when explored through $s$ in the global search and we can thus view states in the closed list as obstacles. Note this observation has been exploited in previous heuristic search work on grid worlds \cite{closedListObstacle2021}. 
Interestingly, we found that incorporating the closed list led to no meaningful performance difference so we ignore it.

\section{Experimental Results}
In this section, we provide empirical evidence demonstrating the benefits of using the proposed data collection framework with LoHA* and the performance of De-LoHA* as an online algorithm. The neural network utilized by all variants and baselines are identical to that described in \citet{localHeuristic}. 


We evaluate our framework on ten 1024x1024 maps with 30\% random obstacles, minimizing travel time between start-goal pairs. Identical to LoHA*, we model a car with state $(x,y,\theta,v)$. The positions are discretized $x,y$ by 0.5, heading $\theta$ by 30 degrees, and velocity $v \in \{-1,0,1,2,3\}$. The car has unit-cost actions that follow Ackermann constraints, $\Delta v \in \{-1, 0, 1\}$ and $\Delta$ steering angle $\in \{-60, -30, 0, 30, 60\}$. The global heuristic used for this domain is a scaled variant of the Euclidean distance $h_g = L_2(s, s_{goal})/3$ (as the max velocity is 3). Unless specified, the results are reported for a small local region size of $K=4$ as this was found to be effective in LoHA*. 

\textbf{Data Efficiency}
Table \ref{table:expansions-per} shows how many nodes expanded are required to collect a single training data point as a function of $K$. We see that running local searches scales poorly and requires 100s of expansions for a single data point for $K > 4$. We observe that as $K$ increases, the work for gathering ``Complete" via backtracking increases but then saturates as when $K$ is sufficiently high, the global A* mainly adds data points on states on the solution path and few else (as the search rarely fully explores a local minima not on the path).
On the flip side, incomplete remains constant regardless of $K$ as all states with any successor expanded becomes a data point (as we have partial data about its $LH(s)$ value), including states that are in local minima.


One concern with using incomplete data by backtracking is that it is noisier data compared to the true residuals collected by local A*. 
In order to understand the impact of this quality difference, Fig \ref{fig:offline_performance} plots the performance of LoHA* as a function of the amount of nodes expanded to collect each dataset. The y-axis ``Speed-up" is the multiplicative reduction of nodes expanded to find paths across 100 problems between weighted LoHA* and weighted A* with $w=4$ (e.g. 3 represents LoHA* expanding 1/3 the nodes). 
We see that even though incomplete data (green) contains approximations, we can still effectively use it to learn a heuristic that leads to performance gain in substantially less data. Using a local A* (blue) or only complete data (orange) requires a magnitude more work to get similar performance. We ran an ablation removing downweighing (i.e. not weighting by $\alpha(s,s')$) and found that this reduced performance from 3.9x to 2.4x. 
Note that although LoHA* expands fewer nodes, the model inference introduces a large overhead which results in it taking roughly 2.3 seconds per problem compared to the baseline which takes 0.1 seconds. 


\textbf{Online Performance}
We evaluate how using our data-efficient framework DE-LoHA* enables learning local heuristics online. 
We initially start with a small dataset collected by solving 5 start-goal planning problems using global A* and the backtracking logic. This small dataset is used to train an initial local heuristic model. The learnt model is then used with DE-LoHA* to solve another 5 start-goal problems. The data accumulated while solving these problems is accumulated and then used to retrain and improve the model further (the model is retrained after solving every 5 problems). 
We evaluate online DE-LoHA* run with different $w$ with the corresponding weighted A* baseline. We plot the performance on 50 problems (plotting on just the 5 problems encountered is too noisy). Within 20 encountered start-goal problems (3 iterations of retraining DE-LoHA*), we start to get a non-trivial performance increase. As we continue solving problems and gathering data, improvement increases ($w=16$) or saturates ($w=4,8$). This highlights how DE-LoHA* can perform well from just solving start-goal problems.

\section{Conclusion}
We have demonstrated how we can collect data more efficiently by reasoning about intermediate steps of the ground truth oracle when applied to LoHA*. Additionally, we have shown how this can enable online data collection and performance improvement, all while solely starting start-goal tasks without any explicit data collection phase. 

\textbf{Extensions}
Our data efficient technique can be applied with no or small modifications to other learning methods which require an oracle search method. When learning a cost-to-go heuristic $h(s,s')$, existing work just uses states on the solution path with the oracle's optimal $c(s,s_{goal})$ \cite{learningGlobalHF2011,learningHeuristicA2020}. We observe how any state $s_i$ and an ancestor $s_i'$ in the oracle's search tree results in a valid optimal $c(s_i',s_i)$ data point that can be used as well. We can thus backtrack from $s_i$ to ancestors to collect data. Similarly for learning an expansion delay \cite{learnExpansionDelay2021}, the delays between states $s_i', s_i$ can be used. We hope future work builds on our data efficient framework to decrease computational burden and enable online learning. 

\section*{Acknowledgements} 
R.V. would like to thank G.K. for his motivation and advice. This material is partially supported by NSF Grant IIS-2328671.


\bibliography{aaai24}

\clearpage

\appendix
\setcounter{figure}{0}
\renewcommand{\thefigure}{A\arabic{figure}}
\setcounter{table}{0}
\renewcommand{\thetable}{A\arabic{table}}

\section{Algorithm}

One of the main benefits of our efficient data collection technique is that it requires minimum modification to existing search algorithms. Algorithm \ref{alg:high-level} depicts the psuedocode for an A* search with necessary components in blue. Since states in a search algorithm already keep track of their parent (Line \ref{line:parent}), and as following backpointers is extremely fast/negligible overhead, our backtracking occurs for ``free".

We additionally highlight how this backtracking logic is applicable to any best-first search algorithm. Therefore, we can use this logic with LoHA* that employs Focal Search \cite{focalSearch1982} instead of A*. When used by non-optimal best-first search algorithms, the cost $c(s,s')$ (Line \ref{line:cost}) is no longer optimal. Instead, it becomes the cost associated with the search algorithm, e.g. with weighted A* with $w=10$ it can return a higher $c(s_1, s_2)$ compared to when run with $w=2$. This data can still be very useful as approximate or bounded-suboptimal values.

\begin{algorithm}[!t]
    \caption{A* with Backtracking Data Collection}
    \label{alg:high-level}
    \footnotesize
    \SetKwInOut{Input}{Input}
    \SetKwInOut{Output}{Output}
    \SetKwFunction{dist}{\scriptsize dist}
    
    \Input{States $s$, Edges, Start $s_{start}$, Goal $s_{goal}$, Local Window Size $K$}
    \Output{Path $[s_{start}, s_i, ... s_{Goal}]$, Local Heuristic Data Points $D_i$}
    \SetAlgoLined\DontPrintSemicolon

    \SetKwFunction{proc}{Plan}
    \SetKwProg{myproc}{Procedure}{}{}

    \SetKwFunction{BacktrackDataCollection}{BacktrackDataCollection}
    \SetKwProg{myprocdata}{Procedure}{}{}

    \SetKwFunction{b}{b}
    \SetKwProg{proclh}{Procedure}{}{}

    \vspace{5pt}
    \myproc{\proc{$S,E,s_{start},s_{goal},K$}}{CompletedData, IncompleteData = $\{\}, \{\}$ \;
    OPEN = $\{s_{start}\}$ \; \label{line:a*-start}
    \While{OPEN not empty}{
        $s_{min}$ = OPEN.min() \;
        \If{$s_{min} == s_{goal}$}{
            \KwRet Reverse $[s_{min}, s_{min}$.parent ... $s_{start}]$, IncompleteData, CompleteData \;
        } 
        Successors = Expand($s_{min}$) \;
        $\forall s' \in$ Successors, $s'$.parent = $s_{min}$ 
        {\color{blue} $\triangleright$ Required anyway to backtrack solution path} \; \label{line:parent}

        {\color{blue} \BacktrackDataCollection{$s_{min}$}} \;
        $\forall s' \in$ Successors, OPEN.insert($s'$) \; \label{line:a*-end}
        
    } 
    \KwRet Failure \;
    }

    \vspace{5pt}
    \myprocdata{\BacktrackDataCollection{$s$}}{ \label{line:backtrackfunction}
    $s_{cur}$ = $s$.parent \;
    \While{$s_{cur}$ != null and $s_{cur} \notin $ CompletedData}{
        \If{\dist($s_{cur}, s$) $>$ $K$}{
            CompletedData[$s_{cur}$] = \b($s_{cur}, s$)\;
            Remove $s_{cur}$ from IncompleteData\;
        }
        \Else{
            IncompleteData[$s_{cur}$] = \b($s_{cur}, s$) \;
        }
        $s_{cur} = s_{cur}$.parent \;
    } 
    } 
    
    \vspace{5pt}
    
    \proclh{\b{$s$, $s'$}}{
        $c(s, s') = s'.cost - s.cost$ \; \label{line:cost}
        $h_k = c(s, s') + h_g(s') - h_g(s)$ \;
        \KwRet $h_k$ \;
    } 
\end{algorithm}

\end{document}